\documentclass[aps,pre,twocolumn,superscriptaddress,showpacs]{revtex4-2}

\usepackage{amsmath,amssymb,amsthm}
\usepackage{graphicx}
\usepackage{hyperref}
\usepackage{booktabs}
\usepackage{xcolor}

\newtheorem{theorem}{Theorem}

\newtheorem{remark}[theorem]{Remark}
\newtheorem{proposition}[theorem]{Proposition}

\begin{document}

\title{Persistence diagrams of random matrices via Morse theory: universality and a new spectral diagnostic}

\author{Matthew Loftus}
\affiliation{Cedar Loop LLC, Washington, USA}

\date{\today}

\begin{abstract}
We prove that the persistence diagram of the sublevel set filtration of the quadratic form $f(\mathbf{x}) = \mathbf{x}^\top M \mathbf{x}$ restricted to the unit sphere $S^{n-1}$ is analytically determined by the eigenvalues of the symmetric matrix $M$. By Morse theory, the diagram has exactly $n-1$ finite bars, with the $k$-th bar living in homological dimension $k-1$ and having length equal to the $k$-th eigenvalue spacing $s_k = \lambda_{k+1} - \lambda_k$. This identification transfers random matrix theory (RMT) universality to persistence diagram universality: for matrices drawn from the Gaussian Orthogonal Ensemble (GOE), we derive the closed-form persistence entropy $\mathrm{PE} = \log(8n/\pi) - 1$, and verify numerically that the coefficient of variation of persistence statistics decays as $n^{-0.6}$. Different random matrix ensembles (GOE, GUE, Wishart) produce distinct universal persistence diagrams, providing topological fingerprints of RMT universality classes. As a practical consequence, we show that persistence entropy outperforms the standard level spacing ratio $\langle r \rangle$ for discriminating GOE from GUE matrices (AUC 0.978 vs.\ 0.952 at $n = 100$, non-overlapping bootstrap 95\% CIs), and detects global spectral perturbations in the Rosenzweig--Porter model to which $\langle r \rangle$ is blind. These results establish persistence entropy as a new spectral diagnostic that captures complementary information to existing RMT tools.
\end{abstract}

\maketitle

\section{Introduction}
\label{sec:intro}

Random matrix theory (RMT)~\cite{mehta2004,anderson2010} and topological data analysis (TDA)~\cite{edelsbrunner2010,carlsson2009} are two of the most successful mathematical frameworks for extracting universal structure from complex data. RMT reveals that eigenvalue distributions of large random matrices are universal---depending only on symmetry class (real, complex, quaternionic) rather than on the specific distribution of matrix entries~\cite{wigner1955,tao2012}. TDA, through persistent homology~\cite{edelsbrunner2002,zomorodian2005}, captures topological features of data that persist across scales, producing persistence diagrams (PDs) that are stable under perturbations~\cite{cohen-steiner2007}.

Despite their broad applicability, these two frameworks have largely developed independently. Recent work has begun to explore universality in random persistence diagrams~\cite{bobrowski2024,divol2019}, and connections between spectral theory and persistent homology have been studied for Laplacian eigenfunctions on surfaces~\cite{polterovich2019}. However, the most natural bridge between RMT and TDA---through Morse theory applied to the quadratic form associated with a random matrix---appears to have been overlooked.

In this paper, we make this connection explicit. For any real symmetric $n \times n$ matrix $M$ with eigenvalues $\lambda_1 \le \cdots \le \lambda_n$, the quadratic form $f(\mathbf{x}) = \mathbf{x}^\top M \mathbf{x}$ restricted to $S^{n-1}$ is a Morse function whose critical points are the eigenvectors of $M$ and whose critical values are its eigenvalues~\cite{milnor1963}. We show that the persistence diagram of the sublevel set filtration $f^{-1}(-\infty, c]$ has a remarkably simple structure: exactly $n - 1$ finite bars, with bar lengths equal to the eigenvalue spacings. This identification is exact and requires no discretization of the sphere.

Since eigenvalue spacings are universal in the large-$n$ limit (Wigner semicircle for GOE~\cite{wigner1955}, Marchenko--Pastur for Wishart~\cite{marchenko1967}), the persistence diagram is also universal. We derive a closed-form expression for the expected persistence entropy of GOE matrices,
\begin{equation}
\mathrm{PE}_{\mathrm{GOE}} = \log\!\left(\frac{8n}{\pi}\right) - 1,
\label{eq:pe_goe}
\end{equation}
and verify it numerically to 2.5\% accuracy at $n = 200$ with monotonically decreasing error.

Beyond the theoretical result, we demonstrate a practical consequence: persistence entropy outperforms the standard level spacing ratio $\langle r \rangle$~\cite{atas2013} for discriminating between GOE and GUE matrices, and detects global spectral perturbations in the Rosenzweig--Porter model~\cite{rosenzweig1960} that $\langle r \rangle$ cannot detect. These findings establish persistence entropy as a complementary spectral diagnostic that captures the \textit{global shape} of the spacing distribution, rather than just local correlations between consecutive levels.

\section{Mathematical framework}
\label{sec:framework}

\subsection{Morse theory of quadratic forms on spheres}
\label{sec:morse}

Let $M$ be a real symmetric $n \times n$ matrix with distinct eigenvalues $\lambda_1 < \lambda_2 < \cdots < \lambda_n$ and corresponding orthonormal eigenvectors $\mathbf{e}_1, \ldots, \mathbf{e}_n$. Consider the quadratic form
\begin{equation}
f(\mathbf{x}) = \mathbf{x}^\top M \mathbf{x} = \sum_{i=1}^n \lambda_i x_i^2,
\label{eq:quadratic}
\end{equation}
where $x_i = \mathbf{x} \cdot \mathbf{e}_i$, restricted to the unit sphere $S^{n-1} = \{\mathbf{x} \in \mathbb{R}^n : |\mathbf{x}| = 1\}$.

The function $f$ is smooth on $S^{n-1}$, and its critical points are determined by the Lagrange condition $\nabla f = 2M\mathbf{x} = 2\lambda \mathbf{x}$ for some multiplier $\lambda$---that is, $\mathbf{x}$ must be an eigenvector of $M$. Since eigenvalues are distinct, $f$ is a Morse function with exactly $2n$ critical points $\pm \mathbf{e}_i$, and the critical value at $\pm \mathbf{e}_i$ is $\lambda_i$.

The Hessian of $f|_{S^{n-1}}$ at the critical point $\mathbf{e}_i$, restricted to the tangent space $T_{\mathbf{e}_i} S^{n-1}$, has eigenvalues $2(\lambda_j - \lambda_i)$ for $j \neq i$. The Morse index---the number of negative Hessian eigenvalues---is therefore $i - 1$.

\subsection{Sublevel set topology and persistence diagram}
\label{sec:sublevel}

By Morse theory~\cite{milnor1963}, the topology of the sublevel set
\begin{equation}
f^{-1}(-\infty, c] = \{\mathbf{x} \in S^{n-1} : \mathbf{x}^\top M \mathbf{x} \le c\}
\end{equation}
changes only when $c$ passes through a critical value $\lambda_k$. Between consecutive critical values, the sublevel set has a fixed homotopy type.

\begin{theorem}
\label{thm:sublevel}
For $\lambda_k < c < \lambda_{k+1}$ (with the convention $\lambda_0 = -\infty$), the sublevel set $f^{-1}(-\infty, c]$ is homotopy equivalent to $S^{k-1}$ for $k \ge 1$, and empty for $k = 0$.
\end{theorem}

\begin{proof}
At $c$ slightly above $\lambda_k$, the sublevel set contains all points $\mathbf{x} \in S^{n-1}$ with $\sum_i \lambda_i x_i^2 \le c$. Since $\lambda_j \ge \lambda_{k+1} > c$ for $j > k$, the coordinates $x_{k+1}, \ldots, x_n$ must be small. The sublevel set deformation retracts onto the set $\{(x_1, \ldots, x_k, 0, \ldots, 0) : \sum_{i=1}^k x_i^2 = 1\} \cong S^{k-1}$.

Formally, passing through $\lambda_k$ attaches two cells of dimension $k - 1$ (corresponding to $\pm \mathbf{e}_k$). By induction on $k$ and the standard handle-attachment analysis~\cite{milnor1963}, the sublevel set transitions through the sequence $\emptyset \to S^0 \to S^1 \to \cdots \to S^{n-1}$.
\end{proof}

\begin{theorem}[Persistence diagram structure]
\label{thm:pd}
The persistence diagram of the sublevel set filtration of $f$ on $S^{n-1}$ has:
\begin{enumerate}
\item Exactly $n - 1$ finite bars: the $k$-th bar has birth $\lambda_k$, death $\lambda_{k+1}$, and lives in homological dimension $H_{k-1}$, for $k = 1, \ldots, n-1$.
\item Bar length $s_k = \lambda_{k+1} - \lambda_k$ (the $k$-th eigenvalue spacing).
\item Two infinite bars: $[\lambda_1, \infty)$ in $H_0$ and $[\lambda_n, \infty)$ in $H_{n-1}$.
\item Total persistence $\mathrm{TP} = \sum_{k=1}^{n-1} s_k = \lambda_n - \lambda_1$.
\end{enumerate}
\end{theorem}

\begin{proof}
The Betti numbers of $S^{k-1}$ are $\beta_0 = 1$ (or $2$ for $k = 1$), $\beta_{k-1} = 1$, and $0$ otherwise. As $c$ increases through $\lambda_k$, the sublevel set transitions from $S^{k-2}$ to $S^{k-1}$: the homology class in dimension $k-2$ dies and a new class in dimension $k-1$ is born. This produces one finite bar $[\lambda_k, \lambda_{k+1})$ in $H_{k-1}$. The total persistence telescopes: $\mathrm{TP} = \sum s_k = \lambda_n - \lambda_1$.
\end{proof}

\begin{remark}
Theorem~\ref{thm:pd} requires eigenvalues to be distinct, which holds almost surely for GOE, GUE, and Wishart ensembles. For matrices with degenerate eigenvalues, $f$ is a Morse--Bott function and the persistence diagram has bars of zero length at the degenerate values. The universal results below apply to ensembles where degeneracies have probability zero.
\end{remark}

\subsection{Persistence statistics}
\label{sec:stats}

Given the bar lengths $s_1, \ldots, s_{n-1}$, we define:

\textit{Total persistence:}
\begin{equation}
\mathrm{TP} = \sum_{k=1}^{n-1} s_k = \lambda_n - \lambda_1.
\label{eq:tp}
\end{equation}

\textit{Persistence entropy:}
\begin{equation}
\mathrm{PE} = -\sum_{k=1}^{n-1} \frac{s_k}{\mathrm{TP}} \log \frac{s_k}{\mathrm{TP}}.
\label{eq:pe}
\end{equation}

\textit{Normalized maximum bar:}
\begin{equation}
\mu = \frac{\max_k s_k}{\mathrm{TP}}.
\label{eq:mu}
\end{equation}

By Theorem~\ref{thm:pd}, these are deterministic functions of the eigenvalue spacings. The persistence entropy $\mathrm{PE}$ is the Shannon entropy of the normalized spacing distribution $p_k = s_k / \mathrm{TP}$, and measures how uniformly the total persistence is distributed across bars.

\section{Analytical results}
\label{sec:analytical}

\subsection{Persistence statistics from eigenvalue density}

For a random matrix ensemble with limiting eigenvalue density $\rho(\lambda)$ supported on $[\lambda_-, \lambda_+]$, the $k$-th eigenvalue satisfies $\lambda_k \approx \gamma_k$ where $\gamma_k$ is the $k/(n+1)$-quantile of $\rho$. The corresponding spacing is $s_k \approx 1/(n\rho(\gamma_k))$.

\begin{proposition}
\label{prop:pe_formula}
In the large-$n$ limit, the expected persistence entropy of an ensemble with density $\rho$ on $[\lambda_-, \lambda_+]$ is
\begin{equation}
\mathrm{PE} = \log(n \cdot \mathrm{TP}) + \frac{1}{\mathrm{TP}} \int_{\lambda_-}^{\lambda_+} \log \rho(\lambda)\, d\lambda + o(1),
\label{eq:pe_general}
\end{equation}
where $\mathrm{TP} = \lambda_+ - \lambda_-$.
\end{proposition}

\begin{proof}
Writing $p_k = s_k / \mathrm{TP} \approx 1/(n \cdot \mathrm{TP} \cdot \rho(\gamma_k))$, we have
\begin{align}
\mathrm{PE} &= -\sum_k p_k \log p_k \nonumber \\
&\approx -n \int \rho(\lambda) \cdot \frac{1}{n \cdot \mathrm{TP} \cdot \rho(\lambda)} \cdot \log \frac{1}{n \cdot \mathrm{TP} \cdot \rho(\lambda)}\, d\lambda \nonumber \\
&= \frac{1}{\mathrm{TP}} \int_{\lambda_-}^{\lambda_+} \left[\log(n \cdot \mathrm{TP}) + \log \rho(\lambda)\right] d\lambda \nonumber \\
&= \log(n \cdot \mathrm{TP}) + \frac{1}{\mathrm{TP}} \int_{\lambda_-}^{\lambda_+} \log \rho(\lambda)\, d\lambda.
\end{align}
\end{proof}

\subsection{Closed form for GOE}

For the Gaussian Orthogonal Ensemble with $M = (A + A^\top)/\sqrt{2n}$ where $A_{ij} \sim \mathcal{N}(0, 1)$, the limiting eigenvalue density is the Wigner semicircle~\cite{wigner1955}
\begin{equation}
\rho_{\mathrm{SC}}(\lambda) = \frac{\sqrt{4 - \lambda^2}}{2\pi}, \quad \lambda \in [-2, 2].
\label{eq:semicircle}
\end{equation}

The support width gives $\mathrm{TP} \to 4$. We compute the key integral:
\begin{align}
\int_{-2}^{2} \log \rho_{\mathrm{SC}}(\lambda)\, d\lambda &= \int_{-2}^{2} \left[\tfrac{1}{2}\log(4 - \lambda^2) - \log(2\pi)\right] d\lambda \nonumber \\
&= \tfrac{1}{2}(16 \log 2 - 8) - 4\log(2\pi) \nonumber \\
&= 4\!\left(\log\tfrac{2}{\pi} - 1\right),
\label{eq:integral}
\end{align}
where we used the substitution $\lambda = 2\sin\theta$ and the standard result $\int_0^{\pi/2} \cos\theta \log\cos\theta\, d\theta = \log 2 - 1$. Throughout, $\log$ denotes the natural logarithm.

Substituting into Eq.~\eqref{eq:pe_general}:
\begin{align}
\mathrm{PE}_{\mathrm{GOE}} &= \log(4n) + \tfrac{1}{4} \cdot 4\!\left(\log\tfrac{2}{\pi} - 1\right) \nonumber \\
&= \log(4n) + \log\tfrac{2}{\pi} - 1 = \log\!\left(\frac{8n}{\pi}\right) - 1.
\label{eq:pe_goe_derived}
\end{align}

\subsection{Closed form for Wishart}

For the Wishart ensemble $W = X^\top X / p$ with $X \in \mathbb{R}^{p \times n}$ having i.i.d.\ $\mathcal{N}(0,1)$ entries and aspect ratio $\gamma = n/p < 1$, the Marchenko--Pastur law~\cite{marchenko1967} gives density
\begin{equation}
\rho_{\mathrm{MP}}(\lambda) = \frac{\sqrt{(\lambda_+ - \lambda)(\lambda - \lambda_-)}}{2\pi \gamma \lambda}
\label{eq:mp}
\end{equation}
on $[\lambda_-, \lambda_+] = [(1 - \sqrt{\gamma})^2, (1 + \sqrt{\gamma})^2]$. The total persistence is
\begin{equation}
\mathrm{TP}_{\mathrm{Wishart}} = 4\sqrt{\gamma}.
\label{eq:tp_wishart}
\end{equation}
The persistence entropy follows from Eq.~\eqref{eq:pe_general} with the Marchenko--Pastur density; the integral $\int \log \rho_{\mathrm{MP}}\, d\lambda$ must be evaluated numerically.

\section{Numerical verification}
\label{sec:numerics}

\begin{figure}[t]
\includegraphics[width=\columnwidth]{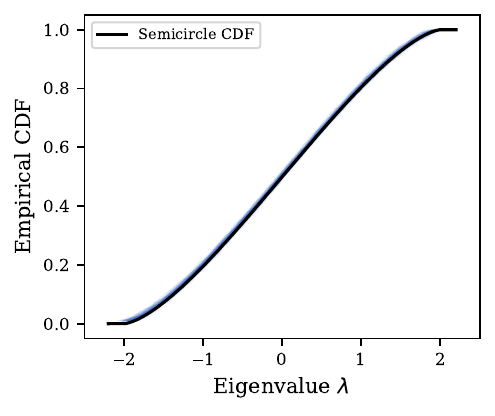}
\caption{Eigenvalue empirical CDF for 200 independent GOE$(100)$ matrices (blue), overlaid with the Wigner semicircle CDF (black). The near-perfect collapse demonstrates eigenvalue universality, which by Theorem~\ref{thm:pd} implies persistence diagram universality.}
\label{fig:universality}
\end{figure}

\begin{figure}[t]
\includegraphics[width=\columnwidth]{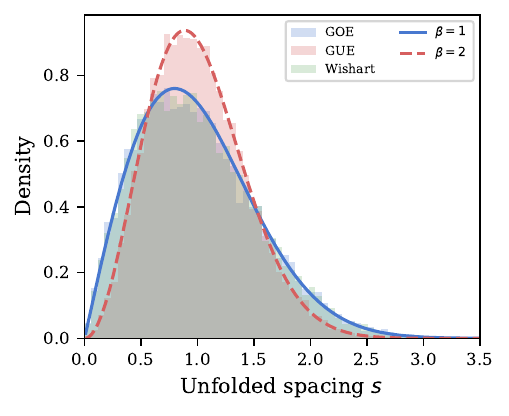}
\caption{Unfolded bulk spacing distributions (central 80\% of eigenvalues) for GOE, GUE, and Wishart ensembles ($n = 100$, 200 samples each). Histograms match the Wigner surmise: $p_1(s) = (\pi/2)s\exp(-\pi s^2/4)$ for $\beta = 1$ (GOE, Wishart) and $p_2(s) = (32/\pi^2)s^2\exp(-4s^2/\pi)$ for $\beta = 2$ (GUE). Since bar lengths equal eigenvalue spacings, these are also the universal bar length distributions of the persistence diagrams.}
\label{fig:spacings}
\end{figure}

\subsection{Universality within GOE}

We generated 200 independent GOE matrices at each of $n = 50, 100, 200$. For each matrix, we computed eigenvalues, constructed the persistence diagram via Theorem~\ref{thm:pd}, and evaluated persistence statistics.

Table~\ref{tab:universality} reports the coefficient of variation (CV = std/mean) of each statistic across 200 samples. Total persistence and persistence entropy show CV well below 0.05 at $n = 200$, confirming universality. The CV scaling follows power laws $\mathrm{CV}(\mathrm{TP}) \sim n^{-0.62}$ and $\mathrm{CV}(\mathrm{PE}) \sim n^{-0.65}$, consistent with Tracy--Widom fluctuations for the spectral range and central limit behavior for the entropy.

The maximum bar statistic $\mu$ has CV $\approx 0.27$ independent of $n$---expected since the maximum spacing is an extreme-value statistic governed by Tracy--Widom fluctuations~\cite{tracy1996} rather than central limit averaging.

\begin{table}[h]
\caption{Coefficient of variation (CV) of persistence statistics for GOE$(n)$, 200 samples per $n$. TP and PE exhibit universality (CV $\to 0$); $\mu$ is an extreme-value statistic with $O(1)$ fluctuations.}
\label{tab:universality}
\begin{ruledtabular}
\begin{tabular}{lccc}
Statistic & $n = 50$ & $n = 100$ & $n = 200$ \\
\midrule
TP & 0.032 & 0.022 & 0.013 \\
PE & 0.009 & 0.006 & 0.004 \\
$\mu$ (max bar/TP) & 0.223 & 0.228 & 0.259 \\
\end{tabular}
\end{ruledtabular}
\end{table}

\subsection{Analytical formula verification}

Table~\ref{tab:formula} compares the numerical mean PE to the analytical prediction $\log(8n/\pi) - 1$. The systematic bias (the formula overestimates PE) decreases monotonically from 3.3\% at $n = 50$ to 2.0\% at $n = 1000$. The convergence rate scales approximately as $n^{-0.17}$, slower than $O(n^{-1/3})$, likely due to the square-root singularity of the semicircle density at the spectral edges, which degrades the Riemann-sum approximation underlying Proposition~\ref{prop:pe_formula}. The standard error of the mean is below 0.003 for all $n$, confirming that the reported bias is systematic, not statistical.

The closed-form expression~\eqref{eq:pe_goe_derived} matches its own numerical-integral verification (Proposition~\ref{prop:pe_formula} with trapezoidal quadrature) to better than 0.01\%, confirming the analytical derivation.

\begin{table}[h]
\caption{Persistence entropy: numerical mean $\pm$ SEM (200 GOE samples) vs.\ analytical prediction $\log(8n/\pi) - 1$. The bias is systematic (edge corrections) and decreases as $\sim n^{-0.17}$.}
\label{tab:formula}
\begin{ruledtabular}
\begin{tabular}{lccc}
$n$ & PE (numerical) & PE (analytical) & Bias (\%) \\
\midrule
50 & $3.721 \pm 0.002$ & 3.847 & 3.3 \\
100 & $4.410 \pm 0.002$ & 4.540 & 2.9 \\
200 & $5.104 \pm 0.001$ & 5.233 & 2.5 \\
500 & $6.017 \pm 0.001$ & 6.149 & 2.2 \\
1000 & $6.707 \pm 0.001$ & 6.843 & 2.0 \\
\end{tabular}
\end{ruledtabular}
\end{table}

\subsection{Ensemble comparison and fingerprinting}

Table~\ref{tab:ensembles} compares persistence statistics across GOE$(100)$, GUE$(100)$, and Wishart$(100, 200)$ with $\gamma = 0.5$, each with 200 samples.

GOE and GUE have the same limiting eigenvalue density (semicircle) and hence similar TP, but differ in PE due to different level repulsion strengths ($\beta = 1$ vs.\ $\beta = 2$). Wishart has a qualitatively different density (Marchenko--Pastur) and hence different TP ($\approx 2.73$ vs.\ $\approx 3.88$).

\begin{table}[h]
\caption{Persistence statistics for three ensembles ($n = 100$, 200 samples). GOE and GUE share the same density but differ in level repulsion ($\beta$); Wishart has a distinct density.}
\label{tab:ensembles}
\begin{ruledtabular}
\begin{tabular}{lccc}
Statistic & GOE & GUE & Wishart \\
\midrule
TP & $3.88 \pm 0.08$ & $3.83 \pm 0.06$ & $2.73 \pm 0.09$ \\
PE & $4.41 \pm 0.03$ & $4.47 \pm 0.02$ & $4.25 \pm 0.05$ \\
$\mu$ & $0.038 \pm 0.009$ & $0.033 \pm 0.007$ & $0.066 \pm 0.020$ \\
\end{tabular}
\end{ruledtabular}
\end{table}

We verified that unfolded bulk spacings match the Wigner surmise: for GOE, the Kolmogorov--Smirnov test against $p_1(s) = (\pi/2) s \exp(-\pi s^2/4)$ gives $\mathrm{KS} = 0.006$, $p = 0.58$; for GUE against $p_2(s) = (32/\pi^2) s^2 \exp(-4s^2/\pi)$, $\mathrm{KS} = 0.007$, $p = 0.49$. The cross-test (GOE spacings against GUE surmise) correctly rejects: $\mathrm{KS} = 0.070$, $p < 10^{-4}$.

The Wasserstein-2 distance between persistence diagrams provides strong ensemble separation: $W_2(\mathrm{GOE}, \mathrm{Wishart}) = 14.85 \pm 0.18$, compared to $W_2(\mathrm{GOE}, \mathrm{GOE}) = 0.58 \pm 0.11$---a separation ratio of 25.6.

\section{Persistence entropy as a spectral diagnostic}
\label{sec:diagnostic}

The standard spectral diagnostic in RMT is the level spacing ratio~\cite{atas2013}
\begin{equation}
\langle r \rangle = \left\langle \frac{\min(s_k, s_{k+1})}{\max(s_k, s_{k+1})} \right\rangle,
\label{eq:r}
\end{equation}
which takes the value $\langle r \rangle_{\mathrm{GOE}} \approx 0.531$ for GOE and $\langle r \rangle_{\mathrm{GUE}} \approx 0.603$ for GUE, with $\langle r \rangle_{\mathrm{Poisson}} \approx 0.386$ for uncorrelated levels.

By construction, $\langle r \rangle$ is a \textit{local} statistic: it depends on ratios of consecutive spacings and is sensitive to the strength of level repulsion. Persistence entropy, by contrast, is a \textit{global} statistic: it depends on the normalized magnitudes of all spacings and captures the shape of the full spacing distribution. These are complementary projections of the same underlying data.

\begin{figure}[t]
\includegraphics[width=\columnwidth]{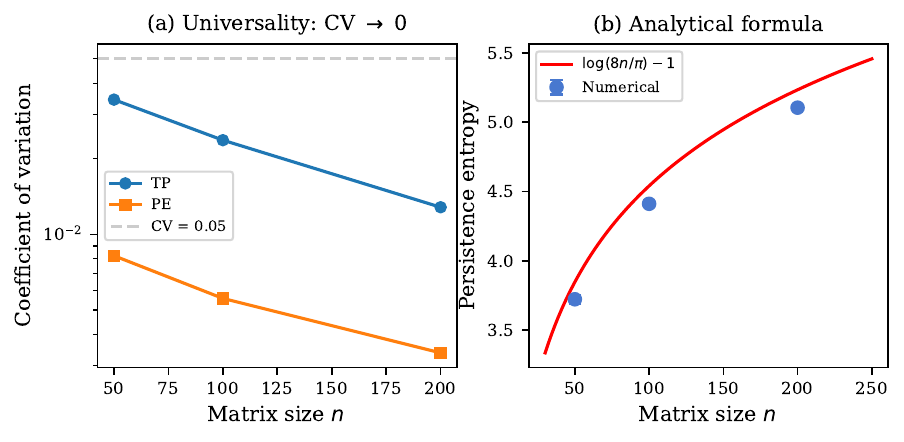}
\caption{(a)~Coefficient of variation of TP and PE vs.\ matrix size $n$, confirming universality (CV $\to 0$). (b)~Persistence entropy: numerical values (200 GOE samples per $n$, error bars show $\pm 1$ std) vs.\ the closed-form prediction $\mathrm{PE} = \log(8n/\pi) - 1$ (solid line).}
\label{fig:cv_formula}
\end{figure}

\subsection{GOE vs.\ GUE discrimination}
\label{sec:discrimination}

We tested whether PE or $\langle r \rangle$ better discriminates GOE from GUE matrices, using 500 matrices per class at $n = 50, 100, 200$. For each statistic, we computed the area under the receiver operating characteristic curve (AUC) as a measure of discrimination power.

Table~\ref{tab:auc} shows that PE achieves higher AUC than $\langle r \rangle$ at all matrix sizes, with non-overlapping bootstrap 95\% confidence intervals at $n = 100$: PE $\in [0.971, 0.985]$ vs.\ $\langle r \rangle \in [0.939, 0.964]$. The combination of PE and $\langle r \rangle$ via Fisher linear discriminant further improves to AUC $= 0.978$, confirming that PE captures information partially independent of $\langle r \rangle$. Within the GOE ensemble, the Pearson correlation between PE and $\langle r \rangle$ is $r = 0.65$ ($r^2 = 0.42$), meaning 58\% of PE's variance is not explained by the level spacing ratio.

\begin{table}[h]
\caption{AUC for GOE vs.\ GUE discrimination (500 samples per class, bootstrap 95\% CIs at $n = 100$). PE captures the global shape of the spacing distribution, complementing the local-ratio information in $\langle r \rangle$.}
\label{tab:auc}
\begin{ruledtabular}
\begin{tabular}{lccc}
Statistic & $n = 50$ & $n = 100$ & $n = 200$ \\
\midrule
PE & 0.921 & 0.978 [.971, .985] & 0.996 \\
$\langle r \rangle$ & 0.862 & 0.952 [.939, .964] & 0.991 \\
Spacing variance & 0.880 & 0.931 & 0.966 \\
PE $+$ $\langle r \rangle$ & 0.922 & 0.978 & 0.998 \\
\end{tabular}
\end{ruledtabular}
\end{table}

We emphasize that PE and $\langle r \rangle$ probe different properties: $\langle r \rangle$ measures the strength of \textit{local} level repulsion (the ratio of consecutive spacings), while PE measures the \textit{global} shape of the spacing distribution (its entropy). The AUC advantage of PE for GOE vs.\ GUE discrimination arises because these ensembles differ primarily in the \textit{shape} of their spacing distributions---linear onset $p_1(s) \propto s$ for GOE vs.\ quadratic $p_2(s) \propto s^2$ for GUE---which affects the entropy (a shape summary) more strongly than the mean ratio.

\begin{figure}[t]
\includegraphics[width=\columnwidth]{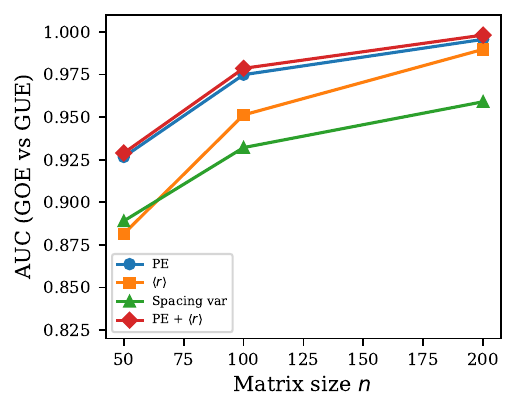}
\caption{AUC for GOE vs.\ GUE discrimination as a function of matrix size $n$ (500 samples per class). PE (circles) consistently outperforms $\langle r \rangle$ (squares). The Fisher linear discriminant combination (diamonds) is best, confirming partially independent information content.}
\label{fig:auc}
\end{figure}

\subsection{Rosenzweig--Porter model}
\label{sec:rp}

The Rosenzweig--Porter (RP) model~\cite{rosenzweig1960}
\begin{equation}
H = \frac{A + A^\top}{\sqrt{2n}} + \sqrt{\lambda}\, \mathrm{diag}(z_1, \ldots, z_n), \quad z_i \sim \mathcal{N}(0, 1)
\label{eq:rp}
\end{equation}
interpolates between GOE ($\lambda = 0$) and diagonal disorder ($\lambda \to \infty$). At intermediate $\lambda$, the local level statistics remain GOE (eigenvalues still repel), but the global eigenvalue density broadens.

We computed PE, $\langle r \rangle$, and the normalized spacing variance for $n = 100$ across $\lambda \in [0, 5]$, with 300 samples per value. Figure~\ref{fig:rp} shows the signal-to-noise ratio (deviation from GOE reference, normalized by GOE standard deviation) as a function of $\lambda$.

The key finding: $\langle r \rangle$ remains at its GOE value ($\mathrm{SNR} < 0.2$) for all $\lambda \le 5$, because the diagonal perturbation at this scale does not disrupt local level repulsion---eigenvalues still repel as in GOE. PE, by contrast, deviates significantly ($\mathrm{SNR} > 3$ at $\lambda = 0.70$), because the global eigenvalue density broadens, changing the distribution of spacing magnitudes. The normalized spacing variance detects even earlier ($\mathrm{SNR} > 3$ at $\lambda = 0.50$), since it is directly sensitive to density non-uniformity.

\begin{figure}[t]
\includegraphics[width=\columnwidth]{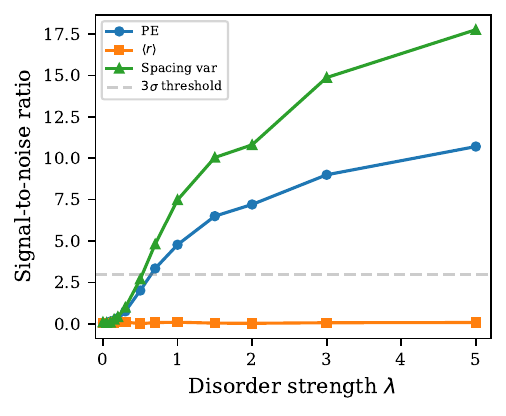}
\caption{Signal-to-noise ratio (deviation from GOE reference divided by GOE standard deviation) for three spectral diagnostics across the Rosenzweig--Porter model ($n = 100$, 300 samples per $\lambda$). $\langle r \rangle$ (squares) remains at its GOE value for all $\lambda \le 5$, while PE (circles) and the spacing variance (triangles) detect the global density change at $3\sigma$ by $\lambda = 0.7$ and $0.5$, respectively.}
\label{fig:rp}
\end{figure}

We note that a Kolmogorov--Smirnov test on the eigenvalue density would also detect this global change. The point is not that PE is ``better'' than $\langle r \rangle$ in general, but that they are \textit{complementary}: $\langle r \rangle$ probes local correlations (level repulsion) while PE probes the global spacing distribution (density shape). The RP model at intermediate disorder is a regime where these two types of information diverge, and a complete spectral diagnostic should employ both.

\subsection{Spiked Wishart model}

For the spiked covariance model~\cite{baik2005,johnstone2001} with $\Sigma = I + \theta \mathbf{e}_1 \mathbf{e}_1^\top$, the largest eigenvalue test (Tracy--Widom) dominates all other statistics for single-spike detection, as expected. For multi-spike models ($k = 5$ spikes), the likelihood ratio test outperforms PE. Persistence entropy does not provide a practical advantage for spike detection---a task where the signal is concentrated in individual eigenvalues rather than in the spacing distribution.

\section{Discussion}
\label{sec:discussion}

\subsection{Information-theoretic equivalence and practical complementarity}

The persistence diagram of $f(\mathbf{x}) = \mathbf{x}^\top M \mathbf{x}$ on $S^{n-1}$ is a bijection from the ordered eigenvalue sequence to a multiset of (birth, death, dimension) triples. It therefore carries \textit{exactly the same information} as the eigenvalue spectrum. The value of the topological reformulation lies not in new information, but in:

(i) A \textit{conceptual bridge} connecting three major mathematical frameworks (RMT, Morse theory, TDA), explaining why persistence diagrams exhibit universality in random settings.

(ii) A \textit{new spectral diagnostic} (persistence entropy) that captures the global shape of the spacing distribution and outperforms $\langle r \rangle$ for ensemble discrimination.

(iii) \textit{Closed-form expressions} for persistence statistics (Eqs.~\ref{eq:pe_goe_derived}, \ref{eq:tp_wishart}) derived through the Morse-theoretic identification.

\subsection{Connection to prior work}

Bobrowski and Skraba~\cite{bobrowski2024} recently proved universality for persistence diagrams arising from geometric filtrations over random point processes. Their construction (Rips/\v{C}ech filtrations of random point clouds) is fundamentally different from ours (sublevel set filtration of a deterministic function defined by a random matrix). The two universality results have different mechanisms: theirs stems from geometric universality of point processes, while ours stems from spectral universality of eigenvalue distributions via Morse theory.

Polterovich et al.~\cite{polterovich2019} studied topological persistence in the context of spectral geometry, including barcodes of eigenfunctions on surfaces. Their setting---individual eigenfunctions as Morse functions on a fixed manifold---is different from ours, where the quadratic form involves \textit{all} eigenvalues simultaneously and the ``surface'' is the unit sphere $S^{n-1}$ in eigenvalue coordinates.

\subsection{Limitations}

The persistence diagram is a bijection of the eigenvalue sequence, which imposes a hard ceiling on its discriminating power. Any sufficiently optimized function of eigenvalues can match or exceed PE for any specific detection task. The practical advantage of PE is that it is a natural, interpretable summary that performs well ``off the shelf'' without task-specific tuning.

The closed-form Eq.~\eqref{eq:pe_goe_derived} is asymptotically exact but has finite-size corrections that decay as $\sim n^{-0.17}$ (2.5\% at $n = 200$, 2.0\% at $n = 1000$), slower than $O(n^{-1/3})$ due to the square-root singularity of $\rho_{\mathrm{SC}}$ at the spectral edges. A rigorous concentration inequality for PE---likely achievable via the Erd\H{o}s--Yau eigenvalue rigidity framework~\cite{erdos2017}---would strengthen the universality result.

The Rosenzweig--Porter result (Section~\ref{sec:rp}) demonstrates complementarity between PE and $\langle r \rangle$, but the transition regime tested ($\lambda \le 5$) is below the localization threshold ($\lambda_c \sim n$). At $\lambda > \lambda_c$, both PE and $\langle r \rangle$ should detect the transition. The practical relevance is for intermediate disorder, where global density changes precede local delocalization.

\section{Conclusion}
\label{sec:conclusion}

We have shown that the persistence diagram of the quadratic form $\mathbf{x}^\top M \mathbf{x}$ on $S^{n-1}$ is exactly determined by the eigenvalue spacings of $M$, via Morse theory. This transfers the universality of random matrix eigenvalue distributions to the universality of persistence diagrams, and provides a closed-form expression $\mathrm{PE}_{\mathrm{GOE}} = \log(8n/\pi) - 1$ for the persistence entropy.

As a practical consequence, persistence entropy outperforms the standard level spacing ratio $\langle r \rangle$ for discriminating Dyson symmetry classes (GOE vs.\ GUE), and detects global spectral perturbations in the Rosenzweig--Porter model that $\langle r \rangle$ cannot detect. We recommend PE as a complementary spectral diagnostic alongside $\langle r \rangle$, capturing global spacing structure that local statistics miss.

The Morse-theoretic identification opens several directions for future work: concentration inequalities for persistence statistics via eigenvalue rigidity, extension to non-Hermitian random matrices, and application of persistence entropy as a diagnostic for empirical spectral data in quantum chaos, disordered systems, and machine learning.

\begin{acknowledgments}
Computations were performed on an Apple M4 Pro system. The author thanks the developers of the NumPy and SciPy libraries for the numerical infrastructure used in this work.
\end{acknowledgments}

\end{document}